\RequirePackage[2022/05/31]{latexrelease}
\documentclass[letterpaper, 10 pt, conference]{ieeeconf}  

\IEEEoverridecommandlockouts                              

\usepackage{graphicx}

\usepackage{amsmath}
\usepackage{amsfonts}
\usepackage{subcaption}
\usepackage{nidanfloat}
\usepackage{multirow}
\usepackage{bm}
\usepackage{cite}
\usepackage{amsthm}

\newtheorem{theorem}{Theorem}[section]

\title{\LARGE \bf
Cooperative Transportation using Multiple Single-Rotor Robots and Decentralized Control for Unknown Payloads
}

\author{Koshi Oishi$^{1}$, Yasushi Amano$^{1}$, and Tomohiko Jimbo$^{2}$
\thanks{$^{1}$The authors are with the Cloud Informatics Research-Domain,
        Toyota Central R\&D Labs., Inc., 41-1, Yokomichi, Nagakute, Aichi, Japan
        {\tt\small e1616@mosk.tytlabs.co.jp}}%
\thanks{$^{2}$The author R-Frontier Division, Frontier Research Center. Toyota Motor Corporation,
        1, Toyota-cho, Toyota, Japan}%
}

\begin{document}

\maketitle
\thispagestyle{empty}
\pagestyle{empty}

\begingroup
  \renewcommand\thefootnote{}            
  \footnotetext{\footnotesize
  © 2025 IEEE. Personal use of this material is permitted. 
  Permission from IEEE must be obtained for all other uses, in any current or future media, 
  including reprinting/republishing this material for advertising or promotional purposes, 
  creating new collective works, for resale or redistribution to servers or lists, 
  or reuse of any copyrighted component of this work in other works.}
  \addtocounter{footnote}{-1}            
\endgroup

\begin{abstract}
Cooperative transportation via multiple aerial robots has the potential to support various payloads and reduce the chances of them being dropped.
Furthermore, autonomously controlled robots render the system scalable with respect to the payload.
In this study, 
a cooperative transportation system was developed using rigidly attached single-rotor robots,
and a decentralized controller was proposed to guarantee asymptotic stability of the error dynamics for unknown strictly positive real systems.
A feedback controller was used to transform unstable systems into strictly positive real ones considering the shared attachment positions.
First, the cooperative transportation of unknown payloads with different shapes larger than the carrier robots was investigated via numerical simulations. 
Second, cooperative transportation of an unknown payload (with a weight of approximately $2.7$ kg and maximum length of $1.6$ m) was demonstrated using eight robots, even under robot failure. 
Finally, the proposed system was shown to be capable of carrying an unknown payload,
even if the attachment positions were not shared, that is, even if asymptotic stability was not strictly guaranteed.
\end{abstract}

\section{INTRODUCTION}
Recently, aerial robots have been used in various fields such as photography, agriculture, and transportation\cite{c1,c2,c3}.
These robots can move in three dimensions;
thus, their dependence on the construction of ground infrastructure is eliminated,
and transportation via aerial robots can be used to provide a wide range of services \cite{c4}.
However,
the payload that can be transported by a single aerial robot is a major limiting factor,
and control is lost if an aerial robot fails.
This limitation can be overcome by implementing a cooperative transportation system using multiple aerial robots.
Such a system is expected to relax payload limitations, including shape and mass, and reduce the chance of a load being dropped.
Furthermore, the autonomous operation of each robot can facilitate the plugin/out of robots, which improves the expandability of the system.

Cooperative transportation employing multiple aerial robots has been widely reported\cite{c5,c6}. System configurations based on cooperative transportation may be divided into two main categories:
cable-suspended \cite{c7,c8,c9,c10,c11,c13} and rigidly connected\cite{c13_5,c14,c15,c16}.
Wehbeh et al. \cite{c13} also proposed a system using rigid rods and ball joints, 
which is similar to cable-suspended configurations.
Studies on cable-suspended configurations have primarily focused on robot formations \cite{c9} and leader tracking \cite{c10,c11}.
Gassner et al. \cite{c10} proposed a decentralized control system where one follower detected the leader using a camera, 
and this method was verified using two aerial robots.

In contrast, rigidly connected configurations are simpler than cable-suspended configurations,
and the probability of aerial robot collisions is reduced.
However, because the robots hold a payload directly,
coordination in rigidly connected configurations is more difficult than in cable-suspended ones \cite{c13_5}.
Mellinger et al. \cite{c13_5} proposed a decentralized controller 
and demonstrated it using aerial robots and a lightweight payload. 
In addition, 
Wang et al. \cite{c16} proposed a decentralized controller limited to centrosymmetric formations;
however, the effectiveness of this approach was confirmed via simulations 
and not using actual robots.
Owing to the constraints of rigidly connected configurations,
only a few experiments have been conducted with actual robots.

\begin{figure}[t]
        \begin{minipage}[b]{1\linewidth}
        \centering
        \includegraphics[keepaspectratio, scale=0.48]{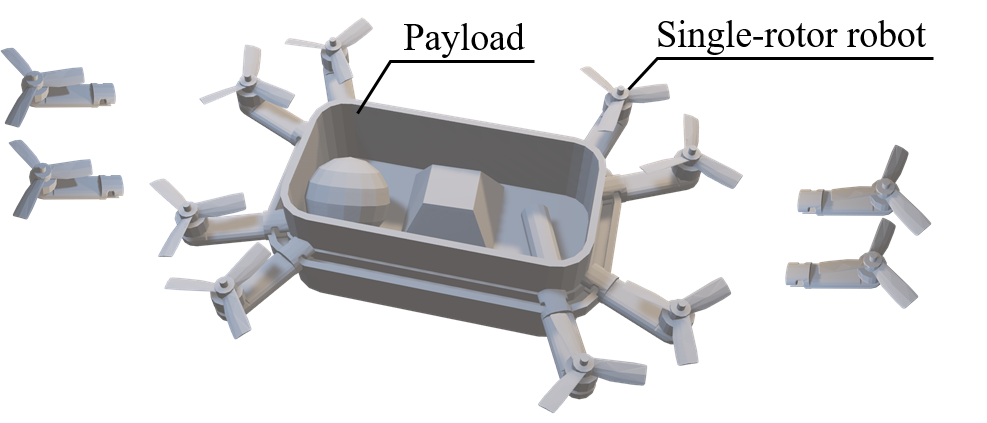}
        \subcaption{Rigidly connected robots to a payload}\label{fig:1a}
        \end{minipage}\\
        \begin{minipage}[b]{1\linewidth}
        \vspace{3mm}
        \centering
        \includegraphics[keepaspectratio, scale=0.36]{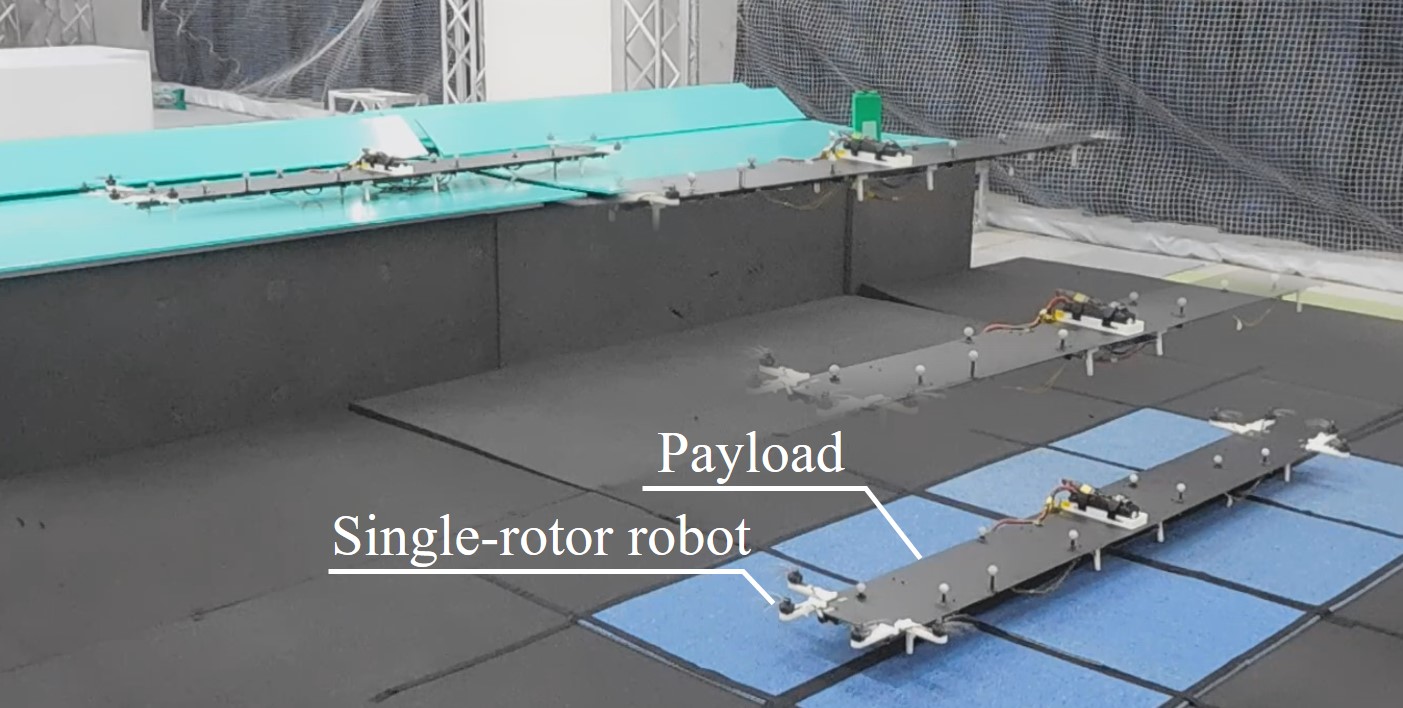}
        \subcaption{Decentralized controlled flight}\label{fig:1b}
        \end{minipage}
        \caption{Cooperative transportation system with autonomously controlled robots}\label{fig:1}
        \vspace{-4mm}
\end{figure}

In our previous study \cite{c16_5},
we proposed a cooperative transportation system with multiple single-rotor robots, which did not possess the ability to fly by themselves. 
The single-rotor robots communicated with each other to estimate the mass and the position of the center of mass (COM) of the payload to be transported. 
Consequently, the mixing matrix was derived and shared to design a centralized flight controller for stabilization of the transportation system.
In contrast, this study investigates decentralized control for a cooperative transportation system with rigid connections, as shown in Fig. \ref{fig:1a}.
The proposed method eliminates the requirements of a mixed matrix or communication to estimate the COM.
Instead, 
by sharing the attachment point, 
each robot derives a controller that is robust to fluctuations in the payload (mass and COM) and robot failures to achieve stabilization.
Furthermore, from a practical perspective, it is desirable to prove the asymptotic stability of the controller relative to a reference.
Therefore, we extend the autonomous smooth switching controller (ASSC) \cite{c17},
whose asymptotic stability has been proven for single-output systems, to multiple-output systems.
However, aerial transportation systems are unstable and the ASSC is only applicable to strictly positive real (SPR) systems.
Therefore, we introduce a feedback controller that transforms the system into an SPR system by sharing the attachment positions among the robots.
The contributions of this study are as follows:

\begin{itemize}
        \item A novel decentralized controller-based ASSC is proposed for unstable transportation systems. 
        This proposed controller is robust against expected fluctuations (mass, COM, and failures), and the asymptotic stability of multiple outputs to the references is proven.
        \item The effectiveness is confirmed via simulations considering rectangular and L-shaped payloads.
        \item Cooperative transportation is demonstrated using a prototype with an unknown payload larger than the robots, even under robot failure, as shown in Fig. \ref{fig:1b}.
\end{itemize}

The remainder of this paper is organized as follows.
Section II describes the dynamics of the target system.
Section III describes the proposed controller based on ASSC.
Section IV describes the numerical experiments for two types of payloads,
and Section V describes the results of prototype experiments using an elongated payload. Finally, the conclusions are presented in Section VI.

\section{DYNAMICS}
The dynamics of a payload transported by multiple single-rotor robots is modeled considering robot failures. 
Specifically, as shown in Fig. \ref{fig:target}, 
consider a configuration in which eight robots having a single rotor each are attached to a rectangular carrier.
The dynamics in the body frame can be approximated around a hovering condition \cite{c13_5} as

\begin{eqnarray}
        \begin{split}
        m\left[
                \begin{array}{c}
                        \ddot x \\
                        \ddot y \\
                        \ddot z \\                        
                \end{array}
                \right ] =
                &\left [
                \begin{array}{ccc}
                        mg & 0 & 0 \\
                        0 & -mg & 0 \\
                        0 & 0 & 0 \\
                \end{array}
                \right ]
                \left[
                \begin{array}{c}
                        \theta \\
                        \phi \\
                        \psi \\                       
                \end{array}
                \right ] \\
                & + \left[
                \begin{array}{c}
                        0 \\
                        0 \\
                        \sum_{i=1}^{8} \sigma_i u_i\\                       
                \end{array}
                \right ] +
                \left[
                        \begin{array}{c}
                                0 \\
                                0 \\
                                -mg\\                       
                        \end{array}
                \right ] \\
        \bm{J}\left[
                \begin{array}{c}
                        \ddot \phi \\
                        \ddot \theta \\
                        \ddot \psi \\                        
                \end{array}
                \right ] =
                &\left [
                \begin{array}{ccc}
                        \sigma_1 r^{1}_2 & \dots & \sigma_8 r^{8}_2 \\
                        \sigma_1 r^{1}_1 & \dots & \sigma_8 r^{8}_1  \\
                        \sigma_1 d_1 c_q & \dots & \sigma_8 d_8 c_q   \\
                \end{array}
                \right ]
                \left[
                        \begin{array}{c}
                                u_{1} \\
                                \vdots \\
                                u_{8} \\
                        \end{array}
                \right ],
        \end{split}
        \label{eq:dy}
\end{eqnarray}
where $x$, $y$, and $z$ are the three-dimensional positions of the payload on the body frame system; 
$\phi$, $\theta$, and $\psi$ are the roll, pitch, and yaw angles of the payload, respectively;
$u_i (i=1,2,\dots,8)$ is the thrust of robot $i$;
$m$ and $\bm{J}$ are the mass and moment of inertia of a payload, respectively;
$\bm{r^i}( = [r^i_1, r^i_2]^\top)$ $( i=1,2,\dots,8)$ is the vector from the origin of the body frame system, 
COM position of the payload, to the attachment position of robot $i$;
$\sigma_i (\in \{0,1\} (i=1,2,\dots,8))$ is the failure parameter of robot $i$, and $\sigma_i = 0$ when robot $i$ fails;
and $d_i (i=1,2,\dots,8)$ is the rotational direction of robot $i$.
Furthermore, $c_q$ is the thrust torque conversion coefficient and $g$ is the gravitational acceleration.

\begin{figure}[t]
        \vspace{2mm}
        \centering
        \includegraphics[width=80mm]{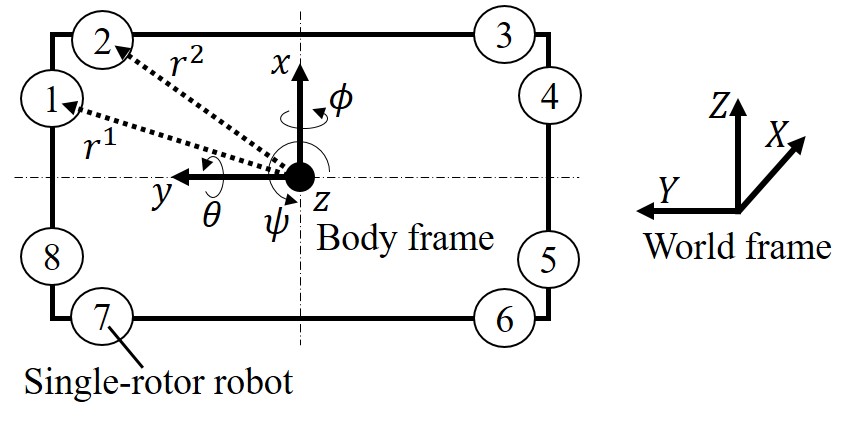}
        \caption{Model of transportation system with multiple rigidly connected robots}
        \label{fig:target}
        \vspace{-4mm}
\end{figure}
We consider the error dynamics for four references, $x_r$, $y_r$, $z_r$, and $\psi_r$ in the body frame.
For the robust feedback controller proposed in the next section, the average thrust of the robots in each quadrant is introduced as
$u_{12} = (u_1 + u_2)/2$, $u_{34}= (u_3 + u_4)/2$, $u_{56}= (u_5 + u_6)/2$, $u_{78}= (u_7 + u_8)/2$.
In this study, the following is assumed:

\begin{itemize}
        \item The robots in each quadrant occupy approximately the same position; in other words, 
        $r^{i_1}=r^{i_2}=r^{i_1 i_2} (i_1 \in \{1,3,5,7\}, i_2=i_1+1)$ for the system with eight robots.
        \item The rotational directions of the robots in each quadrant are the same.
\end{itemize}
Consequently, the error dynamics is obtained from eq. (\ref{eq:dy}) as

\begin{equation}    
        \dot{\bm{\xi}} = \bm{A} \bm{\xi} + \bm{B_{(m,r,\sigma)}}(\bm{U} - \bm{U_r}),
   \label{eq:st}
\end{equation}
where $ \bm{\xi} =
[x_e, \dot x_e, \theta_e, \dot\theta_e, y_e, \dot y_e, \phi_e, \dot \phi_e, z_e, \dot z_e, \psi_e, \dot \psi_e]^\top$,
$x_e = x - x_r$, $y_e = y - y_r$, $z_e = z - z_r$, $\phi_e = \phi - \phi_r$, $\theta_e = \theta - \theta_r$, $\psi_e = \psi - \psi_r$.
$\phi_r$ and $\theta_r$ are the references corresponding to $y_r$ and $x_r$, respectively. 
$\bm{U} (= [u_{12}, u_{34}, u_{56}, u_{78}]^\top)$ is the input vector, and 
$\bm{U_r} (= [u_{r12}, u_{r34}, u_{r56}, u_{r78})]^\top)$ is the stationary input vector.  
$\bm{A} \in \mathbb{R}^{12 \times 12}$ is a state matrix, and
$\bm{B_{(m,r, \sigma)}} \in \mathbb{R}^{12 \times 4}$ is the input matrix depending on the mass of the payload,
$m$, the attachment position vector, $\bm{r}(=[r^{12},r^{34},r^{56},r^{78}]^\top)$, 
and the failure vector $\bm{\sigma}(=[\sigma_{12},\sigma_{34},\sigma_{56},\sigma_{78}]^\top)$.
Note that $\sigma_{i_1 i_2} \in \{1, 2\}(i_1 \in \{1,3,5,7\}, i_2=i_1+1)$.

\section{DECENTRALIZED CONTROL}
In this section, we propose a decentralized controller for aerial transportation systems based on the ASSC using only the broadcasted error, which has been proven to be asymptotically stable when the target system is SPR in \cite{c17}.
However, the target system of eq. (\ref{eq:st}) is not SPR. 
Therefore, first, a robust feedback controller is introduced to transform eq. (\ref{eq:st}) into an SPR system without sharing the exact mass and COM position.
Next, we extend the ASSC to multiple-output systems because its asymptotic stability has been proven for single-output systems \cite{c17}.

\subsection{Robust feedback controller}
The target aerial transportation system is not SPR.
Therefore, 
a state feedback control input $\bm{U_f} (=-\bm{F} \bm{\xi}) \in \mathbb{R}^4$ is added to 
the ASSC input $\bm{U_s} \in \mathbb{R}^4$ to transform the system into an SPR one \cite{c19}. 
Here, $\bm{F} \in \mathbb{R}^{4 \times 12}$ is the feedback gain. 
Consequently, eq. (\ref{eq:st}) is described as

\begin{equation}    
        \dot{\bm{\xi}} = \bm{A} \bm{\xi} + \bm{B_{(m,r,\sigma)}}(\bm{U_f} + \bm{U_s} - \bm{U_r}).
   \label{eq:st2}
\end{equation}
Moreover, the weighted prediction error of $\bm{\xi}$, $\bm{\zeta} \in \mathbb{R}^4$ for four references is introduced as follows: 

\begin{equation}    
        \bm{\zeta} = \bm{C_0} \bm{\xi} + \bm{D_0}(\bm{U_f} + \bm{U_s} - \bm{U_r}),
   \label{eq:st3}
\end{equation}
where $\bm{C_0} \in \mathbb{R}^{4 \times 12}$ and $\bm{D_0} \in \mathbb{R}^{4 \times 4}$.
Note that the introduction of $\bm{C_0}$ and $\bm{D_0}$ corresponds to reducing the relative degree of the transportation system to zero. 
To obtain the feedback gain $\bm{F}$ without sharing the exact $m$ and $r$, even if one fails, 
the matrix $\bm{G} \in \mathbb{R}^{4 \times 4}$ is also designed. 
Thus, $\bm{\zeta}$ of eq. (\ref{eq:st3}) is transformed into

\begin{eqnarray}
        \bm{\eta} = \bm{G} \bm{\zeta},
   \label{eq:st4}
\end{eqnarray}
where $\bm{\eta} (=[\eta_{12}, \eta_{34}, \eta_{56}, \eta_{78}]^\top)$ is the output for the ASSC.
In the following, the subscript of $\bm{B}$ has been omitted for readability.

The target system consisting of eqs. (\ref{eq:st2}) and (\ref{eq:st4}) is SPR if there exist $\bm{F}$, $\bm{G}$,
and positive definite symmetric matrices $\bm{P} \in \mathbb{R}^{12 \times 12}$ that satisfy the following conditions:

\begin{eqnarray}
        \begin{split}
        &\bm{P}(\bm{A}  -  \bm{B} \bm{F})  +  (\bm{A}  -  \bm{B} \bm{F})^\top \bm{P}  +  \\
        & \quad (\bm{P} \bm{B} -  (\bm{G} \bm{C_0})^\top)(\bm{G} \bm{D_0}  + (\bm{G} \bm{D_0})^\top)^{-  1} \times \\
        & \quad \quad (\bm{P} \bm{B}  -  (\bm{G} \bm{C_0})^\top)^\top  <  0,
        \end{split}
   \label{eq:lmi1}
\end{eqnarray}
\begin{eqnarray}
        \bm{G} \bm{D_0}+(\bm{G} \bm{D_0})^\top > 0.
   \label{eq:lmi1-1}
\end{eqnarray}
Because eqs. (\ref{eq:lmi1}) and (\ref{eq:lmi1-1}) cannot be directly solved as a linear matrix inequality (LMI) problem, 
eqs. (\ref{eq:lmi1}) and (\ref{eq:lmi1-1}) are converted into 

\begin{eqnarray}
        \left[ \!
                \begin{array}{cc}
                 \bm{A} \bm{Q} \! + \! \bm{Q} \bm{A}^\top \! \! - \! \bm{B} \bm{R} \! -  \! \bm{R}^\top  \bm{B} & \bm{Q} \bm{C_0}^\top \! \! - \! \bm{B} \bm{S}^\top \\
                 \bm{C_0} \bm{Q} \! - \! \bm{S}  \bm{B}^\top   & -\bm{D_0} \bm{S}^\top \! \! - \! \bm{S} \bm{D_0}^\top \\
                \end{array}
           \!\right] < 0,
           \label{eq:lmi2}
\end{eqnarray}
where is $\bm{Q}(= \bm{P}^{-1}) \in \mathbb{R}^{12 \times 12}$, $\bm{R}(=\bm{F} \bm{Q}) \in \mathbb{R}^{4 \times 12}$,
and $\bm{S}(=\bm{G}^{-1}) \in \mathbb{R}^{4 \times 4}$ can be determined \cite{c19}.

Matrix $\bm{B}$ depends on $m$, $\bm{r}$, and $\bm{\sigma}$.
To transform eq. (\ref{eq:lmi2}) into an SPR system without sharing these exact values,
eq. (\ref{eq:lmi2}) is solved so such that it could be satisfied with multiple $\bm{B}$s simultaneously,
which indicate the vertices of the regions of $m$, $\bm{r}$, and $\bm{\sigma}$.

\begin{figure}[b]
        \centering
        \includegraphics[width=37mm]{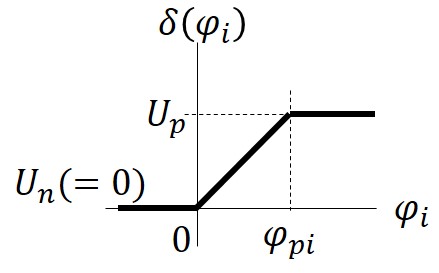}
        \caption{Switching function $\delta (\varphi_i)$}
        \label{fig:us}
\end{figure}

\begin{figure}[t]
        \vspace{2mm}
        \centering
        \includegraphics[width=85mm]{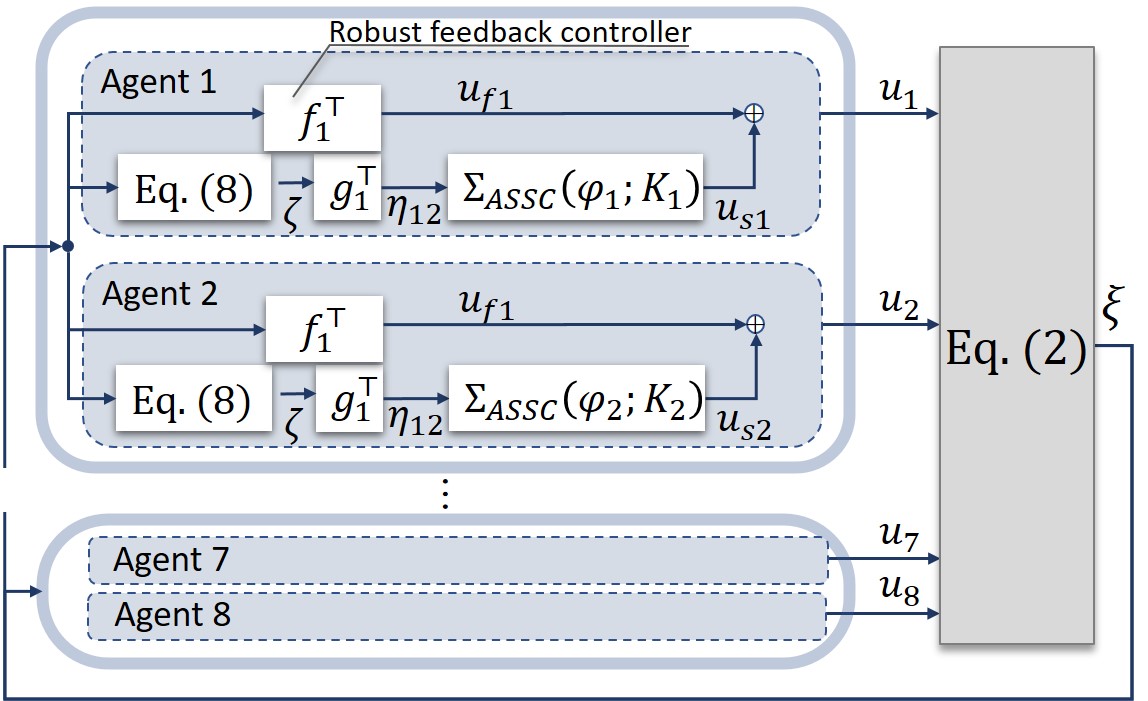}
        \caption{Proposed decentralized control system}
        \label{fig:cont}
        \vspace{-2mm}
\end{figure}

\subsection{ASSC for multiple-output systems}
For the decentralized ASSC, the direct term $\bm{G} \bm{D_0} (\bm{U_f} + \bm{U_s} - \bm{U_r})$ is approximated as

\begin{eqnarray}
        &\bm{\zeta} \simeq \bm{C_0} \bm{\xi} + \hat{\bm{D}}_0 \left[
                \begin{array}{cccc}
                        \ddot{\hat \theta} &\ddot{\hat \phi} & \ddot{\hat z} & \ddot{\hat \psi} \\
                \end{array}
        \right]^\top,
   \label{eq:zeta}
\end{eqnarray}
where $\hat{\bm{D}}_0 \in \mathbb{R}^{4 \times 4}$.
Then, using $\bm{\eta}$ of eq. (\ref{eq:st4}) with eq. (\ref{eq:zeta}), 
the ASSC for multiple outputs is represented as

\begin{eqnarray}
        \begin{split}
        \Sigma_{ASSC} \left\{
                \begin{array}{l}
        \dot\varphi_{i} = -K_{i}\eta_i\\
        u_{si} = \delta(\varphi_{i})=
        \left\{
        \begin{array}{l}
        U_p : \varphi_{i} \ge \varphi_{pi} \\
        \frac{U_p}{\varphi_{pi}} \varphi_i : 0 \le \varphi_{i} < \varphi_{pi} \\
        U_n : \varphi_{i} < 0
        \end{array}
        \right.
        \end{array} \right.
        \end{split}
        \label{eq:cont}
\end{eqnarray}
where $\eta_{j_1} = \eta_{j_2} = \eta_{j_1j_2}$ for $j_1 = \{1,3,5,7\}$ and $j_2 = j_1 + 1$, 
$K_{i} > 0 $ are the gains;
$u_{si}$ is the control input of the ASSC of robot $i$;
$U_p$ and $U_n$ are the upper and lower limits of the thrust applied by the robot, respectively;
and $\delta$ is switching function of $\varphi_i$ as shown in Fig. \ref{fig:us}.
For the rotor without reversal, $U_p$ is the maximum thrust, and $U_n$ is $0$.
$\varphi_{pi}$ is the range of $\varphi_i$ from zero to $U_p$.
Note that the action of robot $i$ can be adjusted via the initial values of $\varphi_{i}$ and $K_{i}$.

For the ASSC in eq. (\ref{eq:cont}) for multiple-output systems, the following theorem holds.

\begin{theorem}
Eq. (\ref{eq:st2}) is assumed to be transformed into a SPR system by $\bm{U_f}.$
Using the ASSC of eq. (\ref{eq:cont}), 
the error $\bm{\eta}$ satisfies $\bm{\eta} \rightarrow \bm{0}$ as $t \rightarrow \infty$.
\end{theorem}

\begin{proof}
For simplicity, consider the case where eq. (\ref{eq:cont}) yields a binary output.
The storage function is defined as

\begin{eqnarray}
        V_{i}^c = \sum_{i=1}^8 \int_0^{\varphi_{i}} \frac{\delta(\varphi_{i}) - u_{ri}}{2K_{i}} d\varphi_{i}
   \label{eq:storage}
\end{eqnarray}
where $K_i$ $(i = 1, \dots, 8)$ are constants, and $u_{ri} \in [U_n,U_p]$.
The derivative of eq. (\ref{eq:storage}) is

\begin{eqnarray*}
        \begin{split}                
        \dot V^c &= \sum_{i=1}^8 \frac{\delta(\varphi_{i}) - u_{ri}}{2K_{i}} \dot{\varphi}_{i} = \sum_{i=1}^8 \frac{u_{ri} - u_{si}}{2} \eta_i\\
                  &= -\bm{\eta}^\top(\bm{U_s} - \bm{U_r})
        \end{split}
   \label{eq:vc}
\end{eqnarray*}
which demonstrates the passivity.
As a result, asymptotic stability of the error is guaranteed for the ASSC. 
Further details regarding the continuous value output in eq. (\ref{eq:cont}) can be found in Amano et al. \cite{c17} using the hyperstability theorem \cite{c31,c32}.
\end{proof}

\subsection{Decentralized control system}
Figure \ref{fig:cont} shows the configuration of the entire control system.
In this study, because the attachment positions were being shared among the robots, each robot could calculate eq. (\ref{eq:lmi2}) by itself.
Here, the expected fluctuations of $m$, $\bm{r}$, and $\bm{\sigma}$ were also shared among robots.
To transform the unstable system into an SPR system, 
the four outputs of the same size as the rank of the input matrix $\bm{B}$ were selected. 
Because our control purpose was to carry a payload to the target position in the target yaw direction, 
$x$, $y$, $z$, and $yaw$ angle were selected. 
As a result of controlling $x$ and $y$, the roll and pitch angles were also controlled. 
Furthermore, $\bm{C_0}$ and $\bm{\hat{D}_0}$ were tuned for each of the four outputs to obtain the prediction error. 

\begin{figure}[t]
        \vspace{-2mm}
        \begin{minipage}[b]{0.5\linewidth}
            \centering
            \includegraphics[keepaspectratio, scale=0.33]{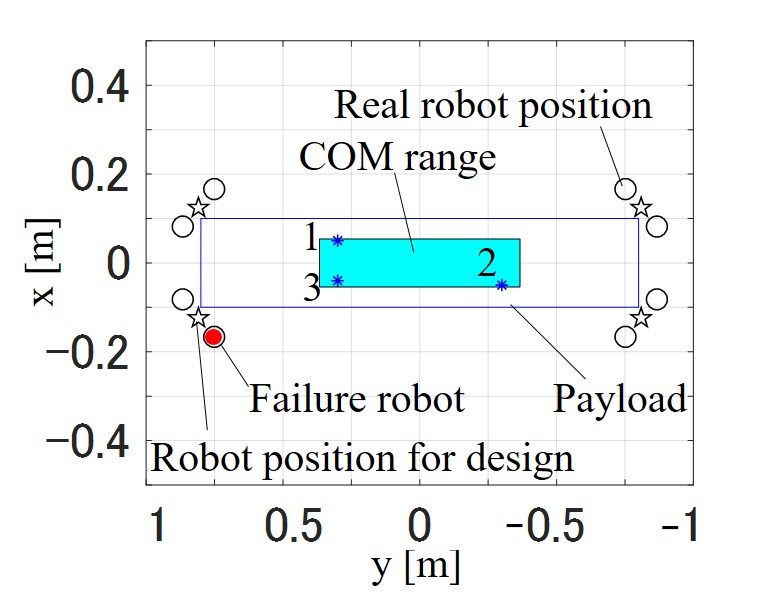}
            \subcaption{Rectangular}\label{fig:ita_con}
        \end{minipage}
        \begin{minipage}[b]{0.49\linewidth}
            \centering
            \includegraphics[keepaspectratio, scale=0.34]{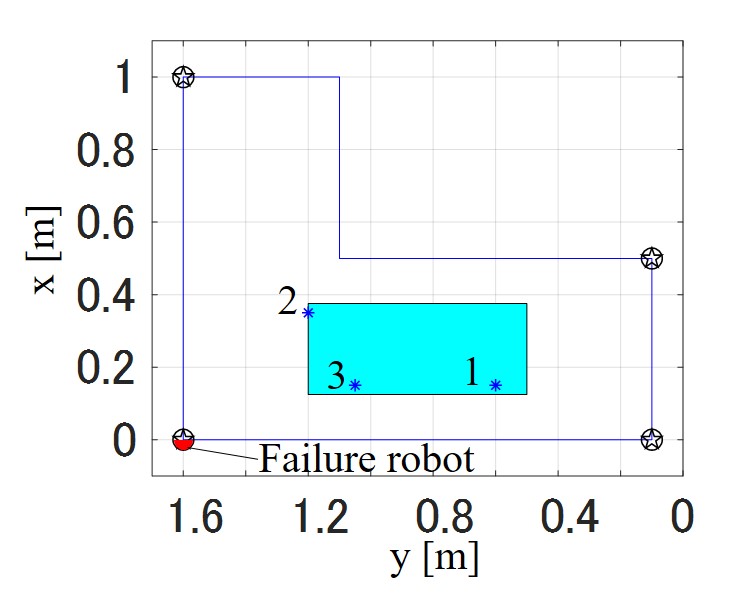}
            \subcaption{L-shaped}\label{fig:L_con}
        \end{minipage}
        \caption{Two payloads with different shapes used for the simulations. 
        The blue line indicates the shape of the payload, 
        the circles indicate the actual robot positions, 
        the stars indicate the robot positions for the control design, 
        the filled area indicates the COM range, 
        the asterisks indicate the simulated COM positions,
        and the red circles indicate the positions of the failed robots.
        }\label{fig:sim_c}
        \vspace{-4mm}
\end{figure}

\begin{figure}[!t]
        \vspace{2mm}
        \begin{minipage}[b]{1\linewidth}
        \centering
        \includegraphics[keepaspectratio, scale=0.373]{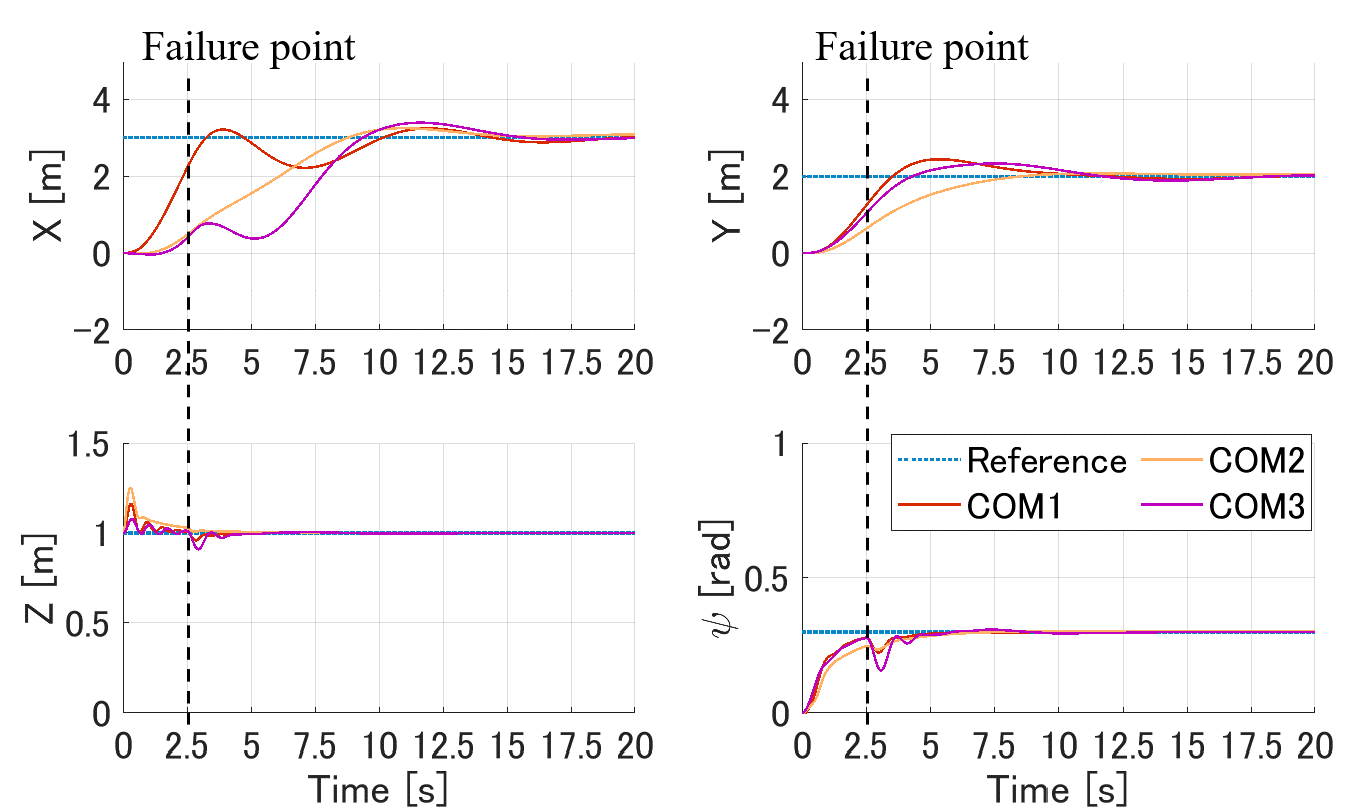}
        \subcaption{Rectangular}\label{fig:sim3}
        \end{minipage} \\
        \begin{minipage}[b]{1\linewidth}
        \centering
        \includegraphics[keepaspectratio, scale=0.373]{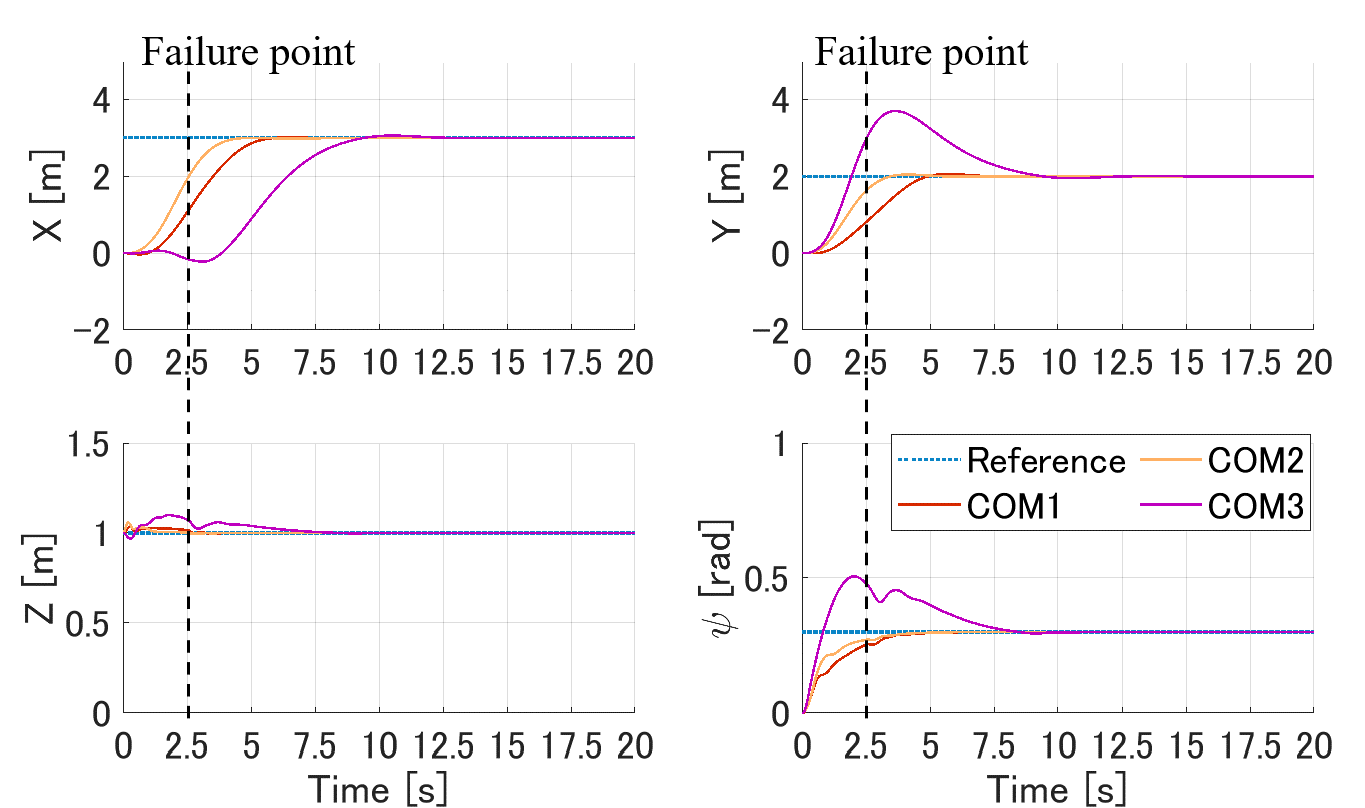}
        \subcaption{L-shaped}\label{fig:sim4}
        \end{minipage}
        \caption{Time series of the position and yaw angle of the payload during three flight simulations with a biased mass and COM, and robot failure.
        The dotted lines represent the reference positions and yaw angle, and the solid lines represent the actual positions and yaw angle.}\label{fig:sim_fail}
        \vspace{-4mm}
\end{figure}

\section{SIMULATION}
Simulations with eight robots were performed using rectangular and L-shaped payloads. 
The controller was designed considering the deviations of the mass and COM position along the $x$ and $y$ axes, as well as robot failures.
For the control design, the equivalent position of the robots in each quadrant was used.
During the simulations, flights with mass and COM bias and robot failures were performed to validate the proposed method.

\subsection{Condition}
Figure \ref{fig:sim_c} shows the two object shapes used for the simulations.
The rectangular-shaped payload is similar to the prototype described in Section V.
In case of the L-shaped payload, the positions of the robots used in the control design and actual design were the same.
The maximum thrust of the single-rotor robot was set to $20$ N.
The expected fluctuations were as follows: $m \in [1 \quad 3]$ kg; the number of failed robots was one; the COM range was as shown in Fig. \ref{fig:sim_c}.
The parameters of the ASSC were as follows: $U_p = 20$ N; $U_n = 0$ N; $K_i = 10 \quad (i=1,\dots, 8)$ for the rectangular payload; $K_i = 20 \quad (i = 1, \dots,8)$ for the L-shaped payload.
In the simulation, 
the flights to the references were performed with three paired masses 
and COM positions $(m, \rm{COM}) \in \{(2,\rm{COM1}), (1, \rm{COM2}), (3,\rm{COM3})\}$, as shown in Fig. \ref{fig:sim_c}.
In addition, one robot was stopped after $2.5$ s to simulate failure.
The references, $X_r$, $Y_r$, and $Z_r$ in the world frame, 
were set to $3.0$, $2.0$, and $1.0$ m, respectively, while $\psi_r$ in the body frame was set to $0.3$ rad.
$X_r$, $Y_r$, and $Z_r$ were transformed into $x_r$, $y_r$, and $z_r$ in the body frame.
The initial position of $z$ was set to $1.0$ m, and the others were set to zero.

\subsection{Results}
Figure \ref{fig:sim_fail} shows the results of the three simulations.
The positions and yaw angle of the payload reached the references,
even if the mass and the COM were biased and robot failure occurred. 
For both shapes, a large fluctuation was observed in the case of COM3. 
This is because the COM position was the closest to the failed robot.
As the L-shaped payload is unbalanced, it tends to fluctuate more than the rectangular one.
The fluctuation in the rectangular payload was found to be larger only in the $X$-direction, 
because there was no margin in the moment arm in the $x$-direction.

\begin{figure}[t]
        \vspace{2mm}
        \begin{minipage}[b]{1\linewidth}
                \centering
                \includegraphics[width=70mm]{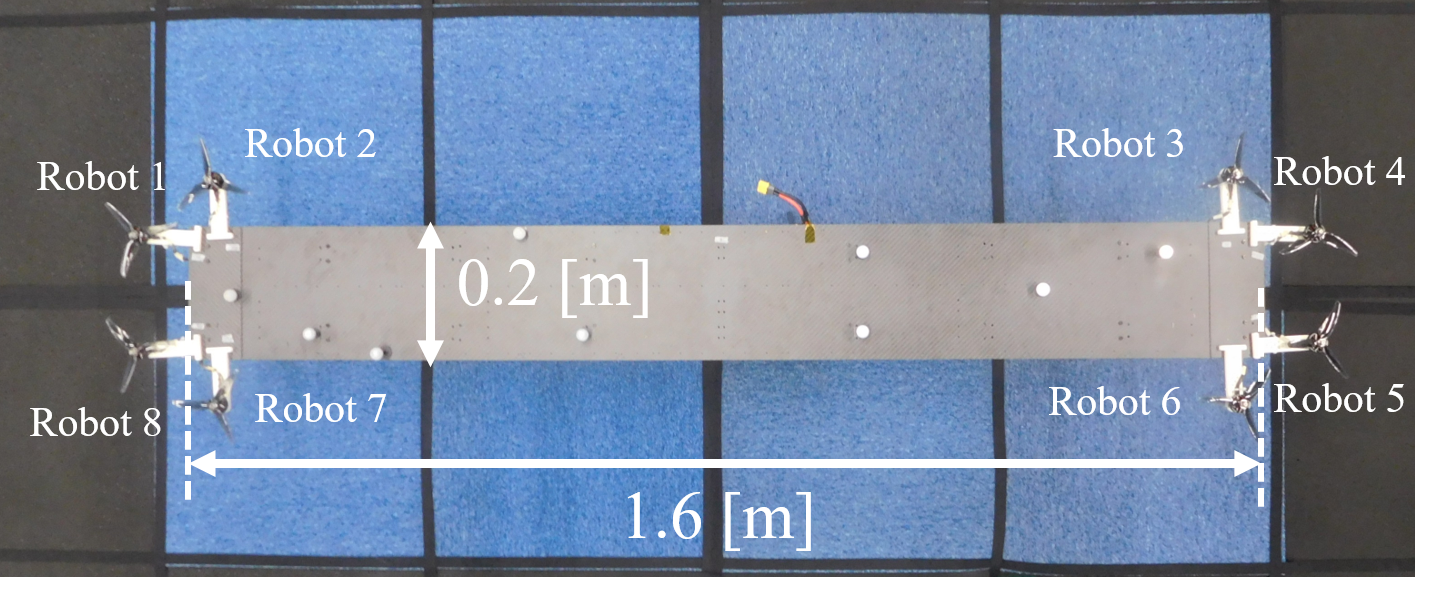}
                \caption{Prototype}
                \label{fig:proto}
        \end{minipage} \\
        \begin{minipage}[b]{1\linewidth}
                \centering
                \includegraphics[width=70mm]{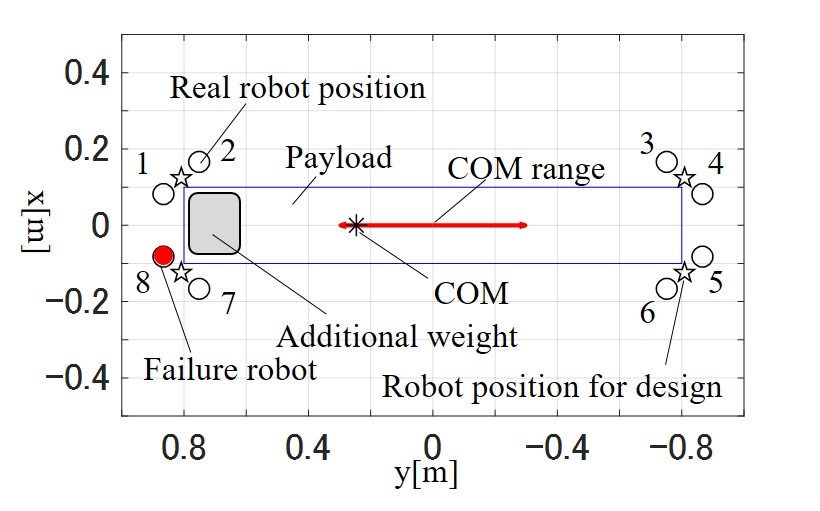}
                \caption{Prototype specifications for the failure experiment. 
                The red arrow denotes the range of the designed COM. 
                The blue line indicates the shape of the payload, 
                the circles indicate the actual robot positions, 
                the stars indicate the robot positions for control design, 
                the filled area indicates the COM range, 
                the asterisks indicate the actual COM, 
                and the red circle indicates the position of the failure robot ($i = 8$).}
                \label{fig:ex_con1}
        \end{minipage}
        \vspace{-4mm}
\end{figure}
\begin{figure*}[b]
        \vspace{-1mm}
        \centering
        \begin{minipage}[b]{0.22\linewidth}
            \centering
            \includegraphics[keepaspectratio, scale=0.33]{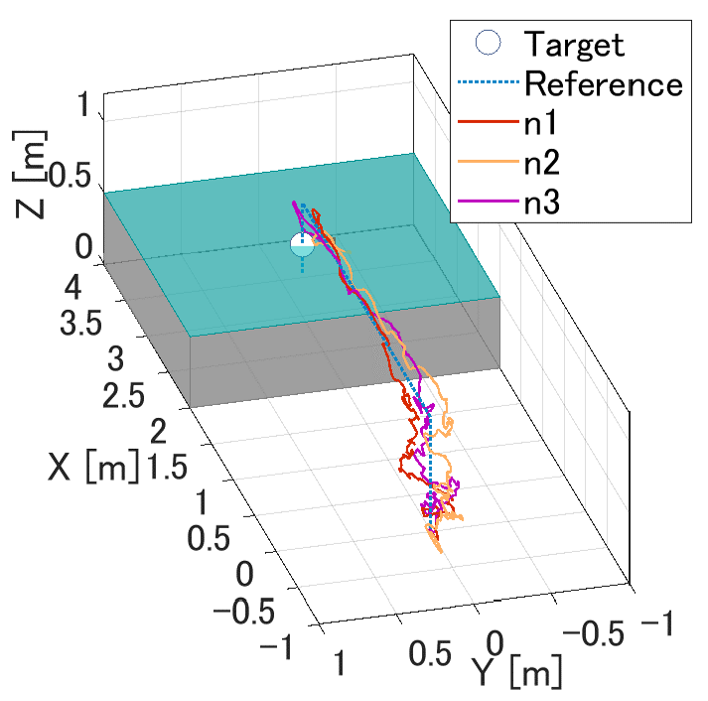}
            \subcaption{Trajectory}\label{fig:ex1l_3D}
        \end{minipage}
        \begin{minipage}[b]{0.38\linewidth}
            \centering
            \includegraphics[keepaspectratio, scale=0.33]{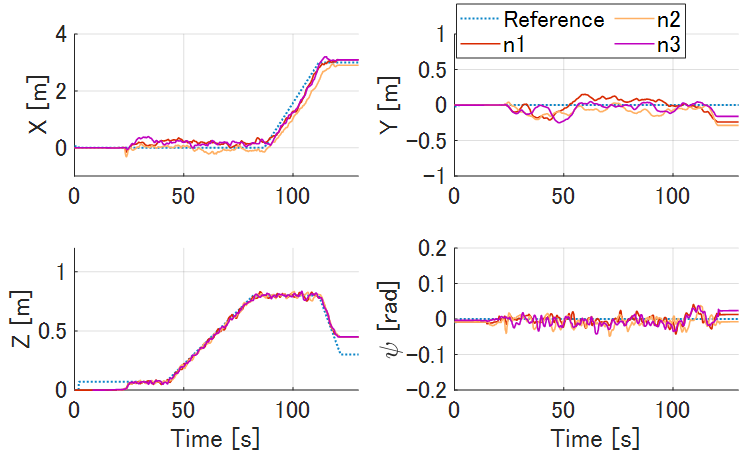}
            \subcaption{Time series of position and yaw angle}\label{fig:ex1l_time}
        \end{minipage} 
        \begin{minipage}[b]{0.38\linewidth}
                \centering
                \includegraphics[keepaspectratio, scale=0.33]{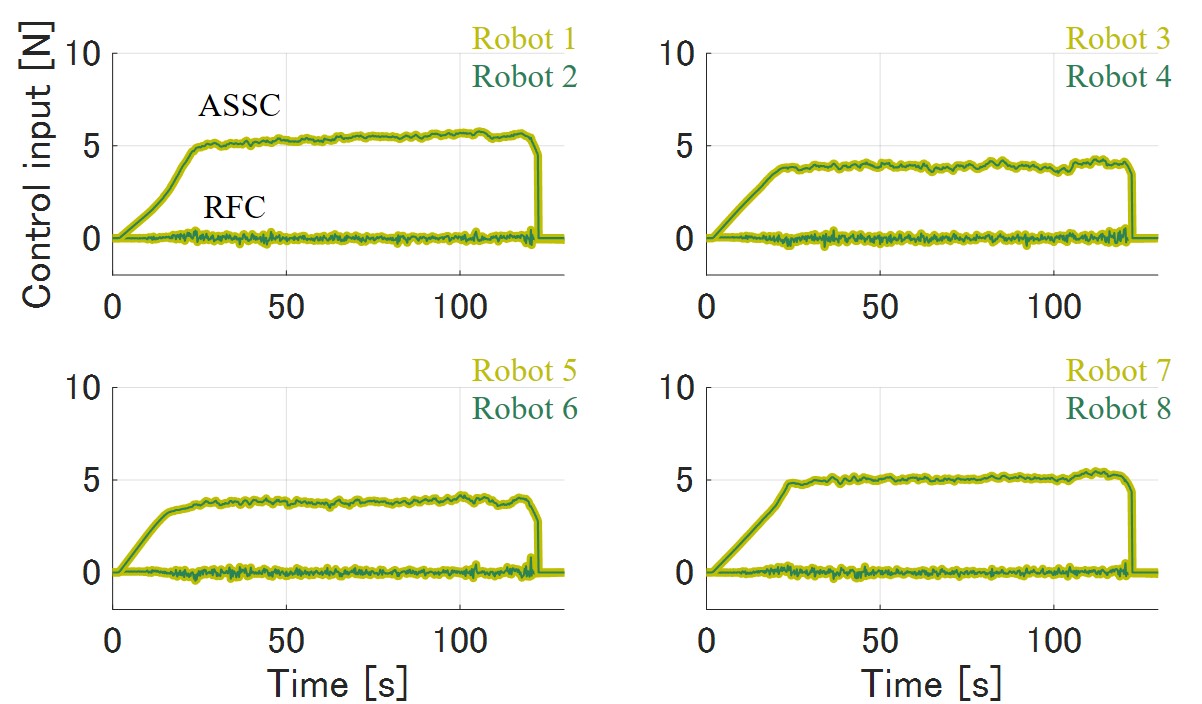}
                \subcaption{Time series of control input}\label{fig:ex1l_U}
            \end{minipage} \\
        \begin{minipage}[b]{0.22\linewidth}
            \centering
            \includegraphics[keepaspectratio, scale=0.33]{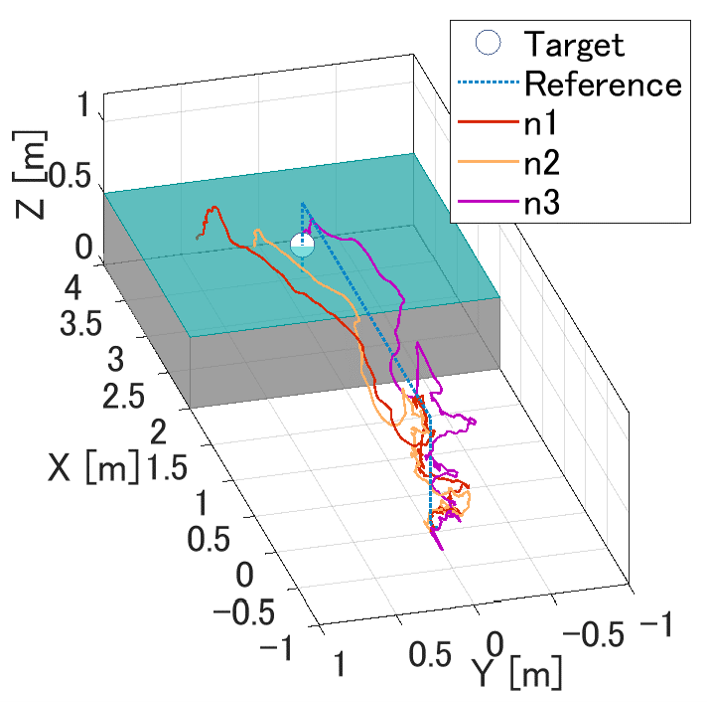}
            \subcaption{Trajectory}\label{fig:ex1r_3D}
        \end{minipage}
        \begin{minipage}[b]{0.38\linewidth}
            \centering
            \includegraphics[keepaspectratio, scale=0.33]{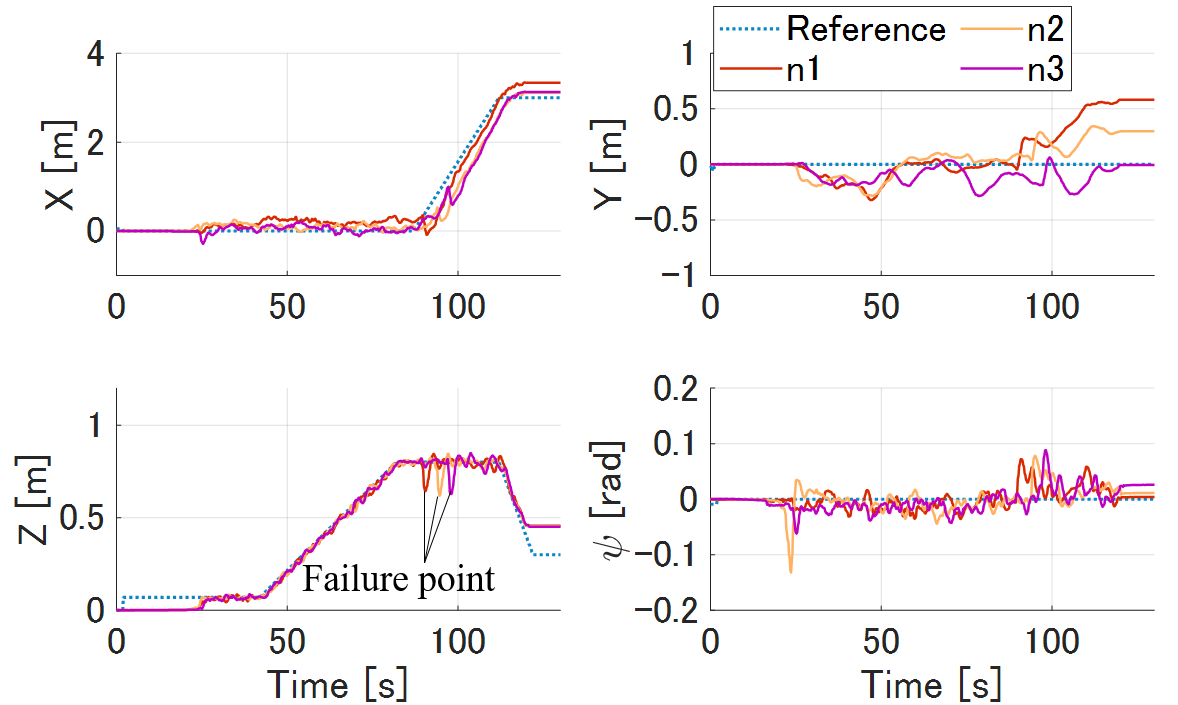}
            \subcaption{Time series of position and yaw angle}\label{fig:ex1r_time}
        \end{minipage}
        \begin{minipage}[b]{0.38\linewidth}
                \centering
                \includegraphics[keepaspectratio, scale=0.33]{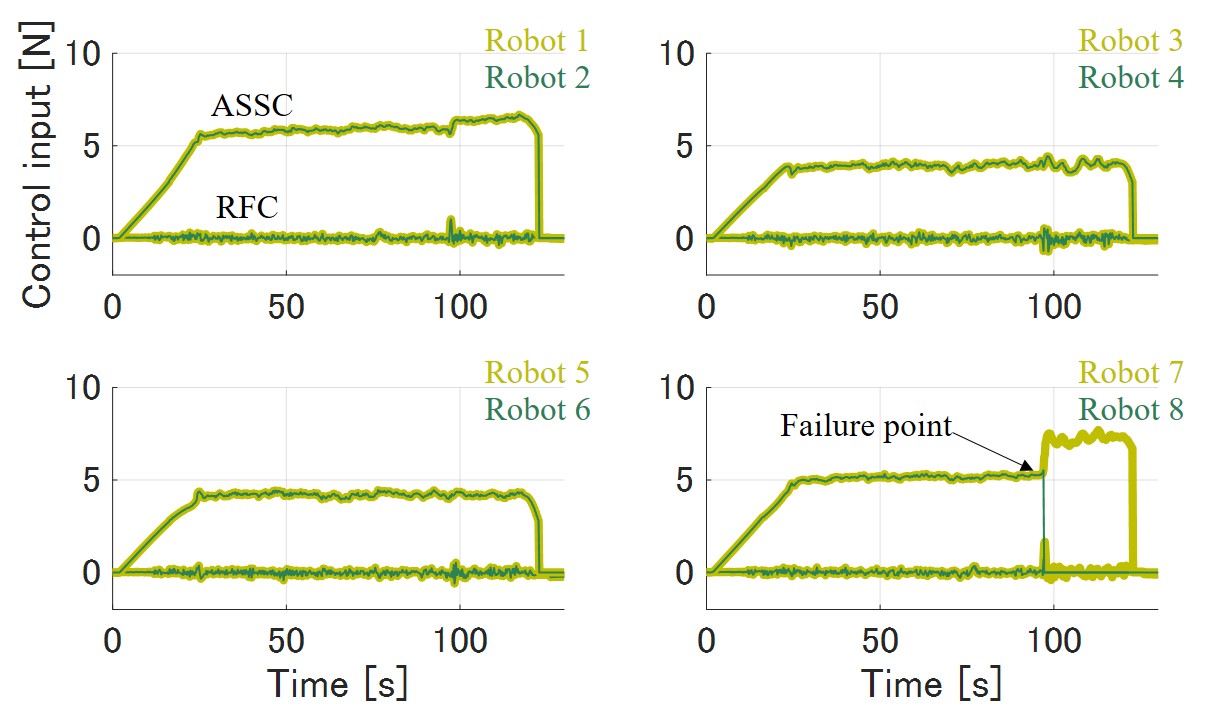}
                \subcaption{Time series of control input}\label{fig:ex2l_U}
            \end{minipage} 
        \caption{Payload positions, yaw angle, and control inputs during three flight experiments with a biased COM. 
        (a), (b), and (c) Biased COM. 
        (d), (e), and (f) Biased COM and robot failure. 
        In (a), (b), (d), and (e), the dotted lines represent the reference trajectory, and the solid lines represent the actual positions during the three flights.
        (c) and (f) indicate the control input of the ASSC and the robust feedback controller (RFC) of n3.
        In (c) and (f), the thick lines represent the odd-numbered robots, and the thin lines represent the even-numbered robots.
        In (f), the failed robot ($i = 8$) does not apply thrust after the failure.
        }\label{fig:ex_result12}
        \vspace{-2.5mm}
\end{figure*}



\section{EXPERIMENT}
To verify our proposed method using actual robots, a prototype was manufactured for use in practical experiments.
The experiment was performed under two conditions.
First,
we verified aerial transportation using both the robust feedback controller and ASSC.
Next, aerial transportation without the robust feedback controller was verified.

\subsection{Prototype}
The prototype consisted of a rectangular payload and eight one-rotor robots, as shown in Fig. \ref{fig:proto}.
Each robot was controlled by a single flight controller (PIXHAWK \cite{c41}) mounted on the transported payload.
Upon implementation of the distributed control, the control input of each robot was calculated independently by the flight controller. 
The software for the method was coded using MATLAB Simulink and Stateflow, and implemented in the flight controller.
In addition, 
the position of the payload was obtained from the motion capture system, while the attitude was obtained via the inertial measurement unit (IMU) in the flight controller.
Further, the COM position and mass of the payload were changed employing the battery and an additional weight.
The size of the payload was $1.6 \times 0.2$ m, and the mass excluding the robot and battery was approximately $1.9$ kg. 
We used $2400$ KV brushless motors and 5.1 inch propellers for the robots.
The maximum thrust of one robot was $14.2$ N according to the motor manufacturer's specification value.


\subsection{Condition}
The controller was designed considering deviations in the COM along the $y$-axis and a robot failure, as shown in Fig. \ref{fig:ex_con1}.
Upon loading the prototype with an additional weight,  
the COM was biased by $0.25$ m in the $y$-direction, with the total mass of the prototype including the additional weight being $2.7$ kg.
In the first transportation experiments, the goal position was a platform that was $0.5$ m high ($Z_r = 0.5$) and $3$ m ahead ($X_r = 3$),
and the reference was considered as a set of waypoints to the goal.

Subsequently, in the failure experiment, one robot was stopped during transportation, as shown in Fig. \ref{fig:ex_con1}.
Finally, in the fully distributed controller without sharing the attachment positions of robots, 
the feedback control input $\bm{U_f}$, shown in Fig. \ref{fig:cont}, was set to zero.
In addition, the robots were only aware of the sign relationship between their inputs and the four outputs.
Here, the experiment was conducted while hovering and the COM was centered.
In all experiments, the measurement data were recorded in a $200$ ms cycle.

\subsection{Results}
Figure \ref{fig:ex_result12} shows the results of aerial transportation with a biased COM and robot failure.
The results confirmed that the proposed method enabled the positions and yaw angle control to reach the reference, 
even if the COM was biased or robot failure occurred.
In the robot failure experiment, 
although the altitude was observed to immediately drop following failure,
the prototype quickly returned to the reference position.
The robust feedback inputs of eight robots were approximately zero except for the unsteady state at the time of failure, 
and are as small as $2$ N or lesser even at the time of failure. 
In other words, most of the flight control is processed by ASSC. 
In addition, ASSC instantly compensated for the shortage of thrust in the event of a failure. 
However, certain offsets were observed in the $Y$-direction at the time of failure.
This offset occurred when the prototype entered a dead band of control owing to fluctuations caused by the failure. 

Figure \ref{fig:ex3} shows the results without the robust feedback input $\bm{U_f}$.
In this case, the prototype took off, hovered, and landed even with fully decentralized control, although the asymptotic stability was not proven.

\begin{figure}[t]
        \vspace{2mm}
        \begin{minipage}[b]{1\linewidth}
        \centering
        \includegraphics[keepaspectratio, scale=0.55]{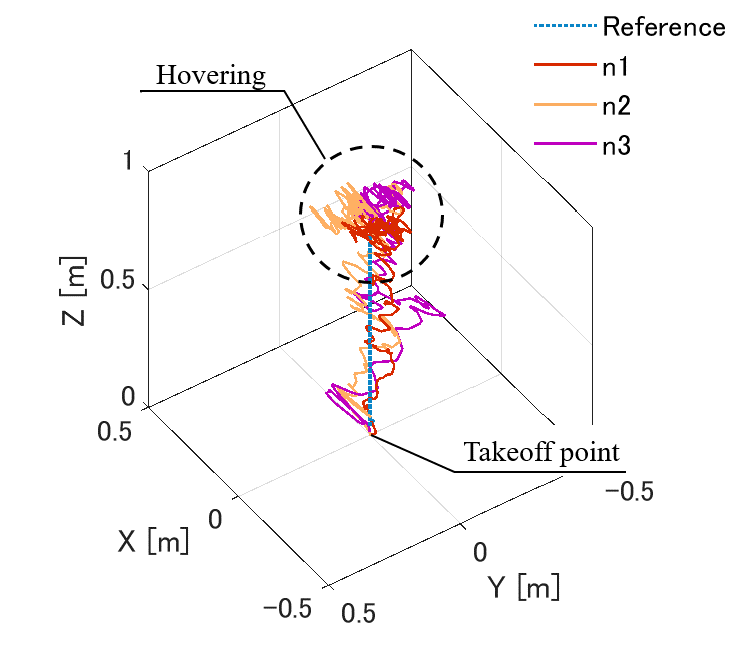}
        \subcaption{Trajectory from takeoff to hovering}\label{fig:ex3_3d}
        \end{minipage} \\
        \begin{minipage}[b]{1\linewidth}
        \vspace{3mm}
        \centering
        \includegraphics[keepaspectratio, scale=0.4]{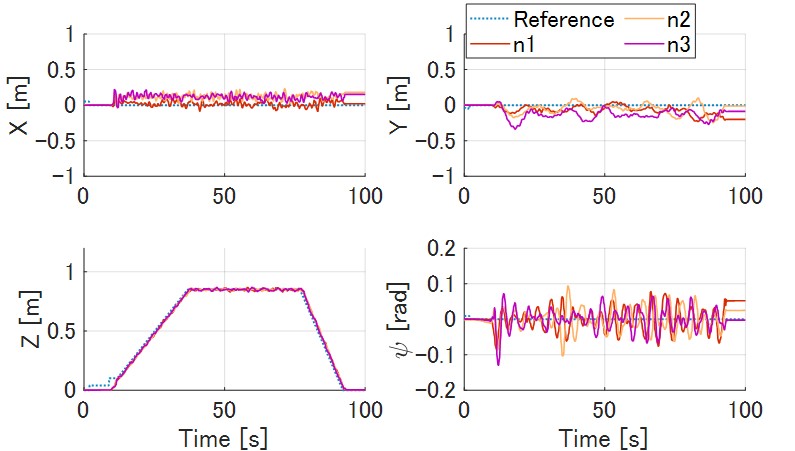}
        \subcaption{Time series of position and yaw angle}\label{fig:ex3_t}
        \end{minipage}
        \caption{Payload positions during three flight experiments without $\bm{U_f}$.
        The dotted lines represent the reference trajectory, and the solid lines represent the actual positions during the three flights.
        }\label{fig:ex3}
        \vspace{-4mm}
\end{figure}

\section{CONCLUSIONS}
In this study, we proposed a novel decentralized controller-based ASSC for unstable transportation systems. 
First, a feedback controller was introduced to transform the unstable system into an SPR system by sharing the attachment positions among the robots. 
This controller was robust against expected fluctuations (mass, COM, and failures). 
Next, the asymptotic stability of the ASSC considering multiple outputs to the references was proven.
Its effectiveness was confirmed through simulations using rectangular and L-shaped payloads. 
Furthermore, cooperative transportation was demonstrated using a prototype with a payload larger than the robots, even under robot failure. 
Finally, the proposed system was shown to be capable of carrying an unknown payload,
even if the attachment positions were not shared. 

Experiments 
confirmed that the system can be hovered only using the ASSC. 
Moreover, being a fully decentralized control system, 
the restrictions on the robot positions are eliminated. 
However, the asymptotic stability using only an ASSC for unstable systems has not yet been proven.
In the future, we aim to realize decentralized control that guarantees asymptotic stability with relaxed restrictions.
Moreover, in this study, distributed processing was performed using a software;
therefore, decentralization on hardware by manufacturing a system wherein each robot has a microcomputer is another area of future research.






\end{document}